\def\year{2019}\relax
\documentclass[letterpaper]{article} 
\usepackage{aaai19}  
\usepackage{times}  
\usepackage{helvet}  
\usepackage{courier}  
\usepackage{url}  
\usepackage{graphicx}  
\usepackage{amsmath, amssymb}
\usepackage{array,color}
\usepackage{graphicx}
\usepackage{verbatim}
\usepackage{booktabs}
\usepackage{algorithm}
\usepackage{algorithmic}
\usepackage{multirow} 
\usepackage{amsmath}
\usepackage{xcolor}
\usepackage{soul}
\usepackage{subfigure}
\usepackage{caption}
\usepackage[hidelinks]{hyperref}
\usepackage{amsthm}
\newtheorem{theorem}{Theorem}

\DeclareMathOperator{\tr}{tr}

\DeclareMathOperator*{\argmax}{arg\,max}

\begin{document}
%
\title{Partial Label Learning with Self-Guided Retraining}
\author{Lei Feng$^{1,2}$\and Bo An$^1$\\
$^1$School of Computer Science and Engineering, Nanyang Technological University, Singapore\\
$^2$Alibaba-NTU Singapore Joint Research Institute, Singapore\\
feng0093@e.ntu.edu.sg, boan@ntu.edu.sg}
\maketitle
\begin{abstract}
Partial label learning deals with the problem where each training instance is assigned a set of candidate labels, only one of which is correct. This paper provides the first attempt to leverage the idea of self-training for dealing with partially labeled examples. Specifically, we propose a unified formulation with proper constraints to train the desired model and perform pseudo-labeling jointly. For pseudo-labeling, unlike traditional self-training that manually differentiates the ground-truth label with enough high confidence, we introduce the maximum infinity norm regularization on the modeling outputs to automatically achieve this consideratum, which results in a convex-concave optimization problem. We show that optimizing this convex-concave problem is equivalent to solving a set of quadratic programming (QP) problems. By proposing an upper-bound surrogate objective function, we turn to solving only one QP problem for improving the optimization efficiency. Extensive experiments on synthesized and real-world datasets demonstrate that the proposed approach significantly outperforms the state-of-the-art partial label learning approaches.
\end{abstract}

\section{Introduction}
In partial label (PL) learning, each training example is represented by a single instance (feature vector) while associated with a set of candidate labels, only one of which is the ground-truth label. This learning paradigm is also termed as \textsl{superset label learning}~\cite{liu2012conditional,liu2014learnability,hullermeier2015superset,gong2018regularization} or \textsl{ambiguous label learning}~\cite{hullermeier2006learning,zeng2013learning,chen2014ambiguously,chen2017learning}. Since manually labeling the ground-truth label of each instance could incur unaffordable monetary or time cost, partial label learning has various application domains, such as web mining~\cite{luo2010learning}, image annotation~\cite{cour2011learning,zeng2013learning}, and ecoinformatics~\cite{liu2012conditional}.

Formally, let $\mathcal{X}\in\mathbb{R}^n$ be the $n$-dimensional feature space and $\mathcal{Y}=\{1,2,\cdots,l\}$ be the corresponding label space with $l$ labels. Suppose the PL training set is denoted by $\mathcal{D}=\{\mathbf{x}_i,S_i\}_{i=1}^m$ where $\mathbf{x}_i\in\mathcal{X}$ is an $n$-dimensional feature vector and $S_i$ denotes the candidate label set. The key assumption of partial label learning lies in that the ground-truth label for $\mathbf{x}_i$ is concealed in its candidate label set $S_i$. The task of partial label learning is to learn a function $f:\mathcal{X}\rightarrow\mathcal{Y}$ from the PL training set $\mathcal{D}$, to correctly predict the label of a test instance.

Obviously, the available labeling information in the PL training set is ambiguous, as the ground-truth label is concealed in the candidate label set. Hence the key for accomplishing the task of learning from PL examples is how to disambiguate the candidate labels. Based on the employed strategy, existing approaches can be roughly grouped into two categories, including the average-based strategy and the identification-based strategy. The average-based strategy assumes that each candidate label makes equal contributions to the model training, and the prediction is made by averaging their modeling outputs~\cite{hullermeier2006learning,cour2011learning,zhang2016partial}. The identification-based strategy considers the ground-truth label as a latent variable, which is identified by an iterative refining procedure~\cite{jin2003learning,nguyen2008classification,liu2012conditional,yu2016maximum}. Although these approaches are able to extract the relative labeling confidence of each candidate label, they fail to reflect the mutually exclusive relationships among different candidate labels.

Motivated by self-training that takes into account such mutually exclusive relationships by directly labeling an unlabeled instance with enough high confidence, this paper gives the first attempt to leverage the similar idea to deal with PL instances. A straightforward method is to first apply a multi-output model on the PL examples, then pick up the candidate label with enough high confidence as the ground-truth label, finally retrain the model on the resulting data. This process is repeated until no PL examples exist, or no PL examples can be picked up as the ground-truth label. Although this method is intuitive, the model learned from PL examples are probably hard to directly identify the ground-truth label in accordance with the modeling outputs, as candidate label sets exist. Furthermore, the incorrectly identified labels could have contagiously negative impacts on the final predictions.

To address this problem, we propose a novel partial label learning approach named SURE (Self-gUided REtraining). Specifically, we propose a novel unified formulation with proper constraints to train the desired model and perform pseudo-labeling jointly. Unlike traditional self-training that manually differentiates the ground-truth label with enough high confidence, we introduce the maximum infinity norm regularization on the modeling outputs to automatically perform pseudo-labeling. In this way, the pseudo labels are decided by balancing the minimum approximation loss and the maximum infinity norm. To optimize the objective function, a convex-concave problem is encountered, as a result of the maximum infinity norm regularization. We show that solving this convex-concave problem is equivalent to solving a set of quadratic programming problems. By proposing an upper-bound surrogate objective function, we turn to solving only one quadratic programming problem for improving the optimization efficiency. Extensive experiments on a number of synthesized and real-world datasets clearly demonstrate the advantage of the proposed approach.
\section{Related Work}
Due to the difficulty in dealing with ambiguous labeling information of PL examples, there are only two general strategies to disambiguate the candidate labels, including the average-based strategy and the identification-based strategy.

The average-based strategy treats each candidate label equally in the model training phase. Following this strategy, some instance-based approaches predict the label $y$ of the test instance $\mathbf{x}$ by averaging the candidate labels of its neighbors~\cite{hullermeier2006learning,zhang2015solving}, i.e., $\argmax_{y\in\mathcal{Y}}\sum_{\mathbf{x}_i\in{\mathcal{N}_{(\mathbf{x})}}}\mathbb{I}(y_i\in S_i)$ where $\mathcal{N}(\mathbf{x})$ denotes the neighbors of instance $\mathbf{x}$. Besides, some parametric approaches adopt a parametric model $F(\mathbf{x}_i,y;\theta)$~\cite{cour2011learning,zhang2016partial} that differentiates the average modeling output of the candidate labels, i.e., $\frac{1}{|S_i|}\sum_{y\in S_i}F(\mathbf{x}_i,y;\theta)$ from that of the non-candidate labels, i.e., $F(\mathbf{x}_i,\hat{y};\theta)$ $(\hat{y}\in\hat{S}_i)$,
where $\hat{S}_i$ denotes the non-candidate label set. Although this strategy is intuitive, an obvious drawback is that the modeling output of the ground-truth label may be overwhelmed by the distractive outputs of the false positive labels.

The identification-based strategy considers the ground-truth label as a latent variable, and assumes certain parametric model $F(\mathbf{x},y;\theta)$ where the ground-truth label is identified by $\argmax_{y\in S_i}F(\mathbf{x}_i,y;\theta)$. Generally, the specific objective function is optimized on the basis of the maximum likelihood criterion: $\max(\sum_{i=1}^m\log(\sum_{y\in S_i}\frac{1}{|S_i|}F(\mathbf{x}_i,y;\theta)))$~\cite{jin2003learning,liu2012conditional} or the maximum margin criterion: $\max(\sum_{i=1}^m(\max_{y\in S_i}F(\mathbf{x}_i,y;\theta)-\max_{\hat{y}\in\hat{S}_i}F(\mathbf{x}_i,\hat{y};\theta)))$~\cite{nguyen2008classification,yu2016maximum}. One potential drawback of this strategy lies in that instead of recovering the ground-truth label, the differentiated label may turn out to be false positive, which could severely disrupt the subsequent model training. 

Self-training is a commonly used technique for semi-supervised learning~\cite{zhu2009introduction}, which is characterized by the fact that the learning process uses its own predictions to teach itself. It has the advantage of taking into account the mutually exclusive relationships among labels by directly labeling an unlabeled instance with enough high confidence. Despite of its simplicity, the early mistakes could be exaggerated by further generating incorrectly labeled data. It is going to be even worse in the PL setting, as the ground-truth label is concealed in the candidate label set.

By alleviating the negative effects of self-training with a unified formulation, a novel partial label learning approach following the identification-based strategy will be introduced in the next section.
\section{The Proposed Approach}
Following the notations in Introduction, we denote by $\mathbf{X} = [\mathbf{x}_1,\cdots,\mathbf{x}_m]^\top\in\mathbb{R}^{m\times n}$ the instance matrix and $\mathbf{Y}=[\mathbf{y}_1,\cdots,\mathbf{y}_m]^\top\in\{0,1\}^{m\times l}$ the corresponding label matrix, where $y_{ij} = 1$ means that the $j$-th label is a candidate label of the $i$-th instance $\mathbf{x}_i$, otherwise the $j$-th label is a non-candidate label. By adopting the identification-based strategy, we also regard the ground-truth label as latent variable, and denote by $\mathbf{P}=[\mathbf{p}_1,\cdots,\mathbf{p}_m]^\top\in[0,1]^{m\times l}$ the confidence matrix where $p_{ij}$ represents the confidence (probability) of the $j$-th label being the ground-truth label of the $i$-th instance.

Unlike self-training that takes into account the mutually exclusive relationships among the candidate labels by performing deterministic pseudo-labeling, we introduce the maximum infinity norm regularization to automatically achieve this consideratum. A unified formulation with proper constraints is proposed as follows:
\begin{gather}
\nonumber
\min \sum_{i=1}^m (L(\mathbf{x}_i,\mathbf{p}_i,f) - \lambda\left\|\mathbf{p}_i\right\|_{\infty})+\beta\Omega(f)\\
\label{eq1}
\text{s.t.}\quad 0 \leq p_{ij} \leq y_{ij},\ \forall i\in[m],\ \forall j\in[l]\\
\nonumber
\sum\nolimits_{j=1}^lp_{ij}=1, \quad \forall i\in [m]
\end{gather}
where $[m]:=\{1,2,\cdots,m\}$, $L$ is the employed loss function, $\Omega$ controls the complexity of model $f$, and $\lambda,\beta$ are the tradeoff parameters. In this unified formulation, the model is learned from the pseudo label matrix $\mathbf{P}$, rather than the original noisy label matrix $\mathbf{Y}$. Besides, unlike the way of traditional self-training to perform deterministic pseudo-labeling by picking up the label with enough high confidence, the confidence of the ground-truth label is differentiated and enlarged by trading off the loss and the maximum infinity norm. Intuitively, only within the allowable range of loss, the candidate label with enough high confidence can be identified as the ground-truth label. In this way, the negative effects of self-training is alleviated by training the model and performing pseudo-labeling jointly. In addition, the first constraint plays two roles: the confidence of each candidate label should be larger than 0, but no more than 1; the confidence of each non-candidate label should be strictly 0. The second constraint guarantees that each confidence vector $\mathbf{p}_i$ will always be in the probability simplex, i.e., $\{\mathbf{p}_i\in[0,1]^l:\sum_{j}p_{ij}=1|i\in[m]\}$. This constraint also implicitly takes into consideration the mutually exclusive relationships among the candidate labels, as the confidence of certain one candidate label is enlarged by the maximum infinity norm regularization, the confidences of other candidate labels will be naturally reduced.

To instantiate the above formulation, we adopt the widely used squared loss, i.e., $L(\mathbf{x}_i,\mathbf{p}_i,f) = \left\|f(\mathbf{x}_i)-\mathbf{p}_i\right\|_2^2$. Besides, we employ the simple linear model $f(\mathbf{x}_i) = \mathbf{W}^\top\mathbf{x}_i+\mathbf{b}$ where $\mathbf{W},\mathbf{b}$
are model parameters. A kernel extension for the general nonlinear case will be introduced in the later section. To control the model parameter, we simply adopt the common squared Frobenius norm of $\mathbf{W}$, i.e., $\left\|\mathbf{W}\right\|_F^2$. To sum up, the final optimization problem is presented as follows:
\begin{gather}
\nonumber
\min\limits_{\mathbf{P},\mathbf{W},\mathbf{b}} \sum_{i=1}^m (\left\|\mathbf{W}^\top\mathbf{x}_i+\mathbf{b}-\mathbf{p}_i\right\|_2^2 - \lambda\left\|\mathbf{p}_i\right\|_{\infty})+\beta\left\|\mathbf{W}\right\|_F^2\\
\label{obfun}
\text{s.t.}\quad 0 \leq p_{ij} \leq y_{ij},\ \forall i\in[m],\ \forall j\in[l]\\
\nonumber
\sum\nolimits_{j=1}^lp_{ij}=1, \quad \forall i\in [m]
\end{gather}
\section{Optimization}
Problem~(\ref{obfun}) can be solved by alternating minimization, which enable us to optimize one variable with other variables fixed. 
This process is repeated until convergence or the maximum number of iterations is reached.
\subsection{Updating $\mathbf{W}$ and $\mathbf{b}$}
With $\mathbf{P}$ fixed, problem (\ref{obfun}) with respective to $\mathbf{W}$ and $\mathbf{b}$ can be compactly stated as follows:
\begin{gather}
\label{eq3}
\min\limits_{\mathbf{W}, \mathbf{b}}\left\|\mathbf{X}\mathbf{W}+\mathbf{1}\mathbf{b}^\top-\mathbf{P}\right\|_F^2+\beta\left\|\mathbf{W}\right\|_F^2
\end{gather}
where $\mathbf{1}$ denotes the vector with all components set to 1. Setting the gradient with respect to $\mathbf{W}$ and $\mathbf{b}$ to 0, the closed-form solutions can be easily obtained:
\begin{align}
\nonumber
\mathbf{W} &= (\mathbf{X}^\top\mathbf{X}+\beta\mathbf{I}-\frac{\mathbf{X}^\top\mathbf{1}\mathbf{1}^\top\mathbf{X}}{m})^{-1}(\mathbf{X}^\top\mathbf{P}-\frac{\mathbf{X}^\top
\mathbf{1}\mathbf{1}^\top\mathbf{P}}{m})\\
\label{eq4}
\mathbf{b} &= \frac{1}{m}(\mathbf{P}^\top\mathbf{1} - \mathbf{W}^\top\mathbf{X}^\top\mathbf{1})
\end{align}
\subsubsection{Kernel Extension}
To deal with the nonlinear case, the above linear learning model can be easily extended to a kernel-based nonlinear model. To achieve this, we utilize a feature mapping $\phi(\cdot):\mathbb{R}^n\rightarrow\mathbb{R}^{\mathcal{H}}$ to map the original feature space $\mathbf{x}\in\mathbb{R}^n$ to some higher (maybe infinite) dimensional Hilbert space $\phi(\mathbf{x})\in\mathbb{R}^{\mathcal{H}}$. By representor theorem~\cite{scholkopf2002learning}, $\mathbf{W}$ can be represented by a linear combination of input variables, i.e., $\mathbf{W} = \phi(\mathbf{X})^\top\mathbf{A}$ where $\mathbf{A}\in\mathbb{R}^{m\times l}$ stores the combination weights of instances. Hence $\phi(\mathbf{X})\mathbf{W} = \mathbf{K}\mathbf{A}$ where $\mathbf{K}\in\phi(\mathbf{X})\phi(\mathbf{X})^\top\in\mathbb{R}^{m\times m}$ is the kernel matrix with each element defined by $k_{ij} = \phi(\mathbf{x}_i)^\top\phi(\mathbf{x}_j)=\kappa(\mathbf{x}_i,\mathbf{x}_j)$, and $\kappa(\cdot,\cdot)$ denotes the kernel function. In this paper, Gaussian kernel function $\kappa(\mathbf{x}_i,\mathbf{x}_j) = \exp(-\left\|\mathbf{x}_i-\mathbf{x}_j\right\|_2^2/(2\sigma^2))$ is employed with $\sigma$ set to the averaged pairwise distances of instances. By incorporating such kernel extension, problem (\ref{eq3}) can be presented as follows:
\begin{gather}
\label{eq5}
\min\limits_{\mathbf{A}, \mathbf{b}}\left\|\mathbf{K}\mathbf{A}+\mathbf{1}\mathbf{b}^\top-\mathbf{P}\right\|_F^2+\beta\tr(\mathbf{A}^\top\mathbf{K}\mathbf{A})
\end{gather}
where $\tr(\cdot)$ is the trace operator. Setting the gradient with respect to $\mathbf{A}$ and $\mathbf{b}$ to 0, the closed-form solutions are reported as:
\begin{align}
\nonumber
\mathbf{A} &= (\mathbf{K}+\beta\mathbf{I}-\frac{\mathbf{1}\mathbf{1}^\top\mathbf{K}}{m})^{-1}(\mathbf{P}-
\frac{\mathbf{1}\mathbf{1}^\top\mathbf{P}}{m})\\
\label{eq6}
\mathbf{b} &= \frac{1}{m}(\mathbf{P}^\top\mathbf{1} - \mathbf{A}^\top\mathbf{K}^\top\mathbf{1})
\end{align}
By adopting this kernel extension, we choose to update the parameters $\mathbf{A}$ and $\mathbf{b}$ throughout this paper.
\subsection{Updating $\mathbf{P}$}
With $\mathbf{A}$ and $\mathbf{b}$ fixed, the modeling output matrix $\mathbf{Q}=[\mathbf{q}_1,\cdots,\mathbf{q}_m]^\top\in\mathbb{R}^{m\times l}$ is denoted by $\mathbf{Q}=\phi(\mathbf{X})\mathbf{W}+\mathbf{1}\mathbf{b}^\top=\mathbf{K}\mathbf{A}+\mathbf{1}\mathbf{b}^\top$, problem (\ref{obfun}) reduces to:
\begin{gather}
\nonumber
\min\limits_{\mathbf{P}} \sum_{i=1}^m (\left\|\mathbf{p}_i-\mathbf{q}_i\right\|_2^2 - \lambda\left\|\mathbf{p}_i\right\|_{\infty})\\
\label{eq7}
\text{s.t.}\quad 0 \leq p_{ij} \leq y_{ij},\ \forall i\in[m],\ \forall j\in[l]\\
\nonumber
\sum\nolimits_{j=1}^lp_{ij}=1, \quad \forall i\in [m]
\end{gather}
Obviously, we can solve problem (\ref{eq7}) by solving $m$ independent problems, one for each example. We further denote by $\textsl{OP}$ the minimum loss of the problem for the $i$-th example:
\begin{gather}
\nonumber
\textsl{OP} = \min\limits_{\mathbf{p}_i}\left\|\mathbf{p}_i - \mathbf{q}_i\right\|_2^2 - \lambda\left\|\mathbf{p}_i\right\|_{\infty}\\
\label{eq8}
\text{s.t.} \quad \mathbf{1}^\top\mathbf{p}_i=1\\
\nonumber
\quad \mathbf{0} \leq \mathbf{p}_i \leq \mathbf{y}_i
\end{gather}
Here, problem (\ref{eq8}) is a constrained convex-concave problem, as the first term is convex while the second term is concave. Instead of using traditional time-consuming convex-concave procedure~\cite{yuille2003concave} to solve this problem, we show that optimizing this problem is equivalent to solving $l$ independent QP problems, each for one label. We denoted by $\textsl{OPI}(j)$ the minimum loss of the problem for the $j$-th label:
\begin{gather}
\nonumber
\textsl{OPI}(j) = \min\limits_{\mathbf{p}_i}\left\|\mathbf{p}_i - \mathbf{q}_i\right\|_2^2 - \lambda p_{ij}\\
\label{eq9}
\text{s.t.} \quad p_{ik}\leq p_{ij},\ \forall k\in[l]\\
\nonumber
\quad \mathbf{1}^\top\mathbf{p}_i=1\\
\nonumber
\quad \mathbf{0} \leq \mathbf{p}_i \leq \mathbf{y}_i
\end{gather}
\begin{theorem}
\label{th1}
$\textsl{OP} = \min\nolimits_{j\in[l]}\textsl{OPI}(j)$.
\end{theorem}
\begin{proof}
It is obvious that there must exist $j\in[l]$ such that $p_{ij} = \left\|\mathbf{p}_i\right\|_{\infty}$ and the optimum loss \textsl{OP} of problem (\ref{eq8}) can be obtained. In addition, if $p_{ij} = \left\|\mathbf{p}_i\right\|_{\infty}$ coincidentally holds, then $\textsl{OPI}(j) = \textsl{OP}$, as in such case, problem (\ref{eq9}) is equivalent to problem (\ref{eq8}). While if $p_{ij}\neq\left\|\mathbf{p}_i\right\|_{\infty}$, then $\textsl{OPI}(j)>\textsl{OP}$. Hence $\textsl{OP} = \min\nolimits_{j\in[l]}\textsl{OPI}(j)$.
\end{proof}
Theorem \ref{th1} gives us a motivation to solve problem (\ref{eq8}) by selecting the minimum loss from $l$ independent quadratic programming problems. However, this may be time-consuming, as the label space could be very large. Thus, we propose a surrogate objective function to upper bound the loss incurred by problem (\ref{eq8}). Specifically, we select the candidate label $j$ with the maximal modeling output by $j=\argmax_{j\in S_i}q_{ij}$ where $S_i$ is the candidate label set containing the indices of candidate labels of the instance $\mathbf{x}_i$. The proposed surrogate objective function is given as:
\begin{gather}
\nonumber
\textsl{OPS} = \min\limits_{\mathbf{p}_i}\left\|\mathbf{p}_i - \mathbf{q}_i\right\|_2^2 - \lambda p_{ij}\\
\label{eq10}
\text{s.t.} \quad q_{ik}\leq q_{ij},\ \forall k\in S_i,\ \exists j\in S_i\\
\nonumber
\quad p_{ik}\leq p_{ij},\ \forall k\in[l]\\
\nonumber
\quad \mathbf{1}^\top\mathbf{p}_i=1\\
\nonumber
\quad \mathbf{0} \leq \mathbf{p}_i \leq \mathbf{y}_i
\end{gather}
Note that the difference between problem (\ref{eq10}) and problem (\ref{eq9}) lies in that problem (\ref{eq10}) adds a constraint to select the label $j$ with the maximal modeling output. Unlike problem (\ref{eq9}) that considers the possibility of each label being the ground-truth label, problem (\ref{eq10}) directly assumes the candidate label $j$ with the maximal modeling output $j=\argmax_{j\in S_i}q_{ij}$ to be the label that is most likely the ground-truth label. This assumption coincides with self-training, which also considers the label with the maximal modeling output as the ground-truth label. Different from self-training that manually performs deterministic pseudo-labeling, our approach aims to automatically enlarge the confidence of the label with the maximal modeling output as much as possible by balancing the two terms in problem (\ref{eq10}). In this way, we not only avoid the opinionated mistakes by self-training, but also take into account the mutually exclusive relationships among candidate labels.
\begin{theorem}
\label{th2}
$\textsl{OP}\leq\textsl{OPS}$.
\end{theorem}
\begin{proof}
From the formulation of problem (\ref{eq10}) and problem (\ref{eq9}), it is easy to see that $\textsl{OPS}\in\{\textsl{OPI}(j)|j\in S_i\}$. Since $S_i\subset[l]$, $\textsl{OPS}\in\{\textsl{OPI}(j)|j\in[l]\}$. Which means, $\min_{j\in[l]}\textsl{OPI}(j)\leq\textsl{OPS}$. Using Theorem (\ref{th1}), $\textsl{OP}=\min_{j\in[l]}\textsl{OPI}(j)\leq\textsl{OPS}$.
\end{proof}
Theorem (\ref{th2}) shows that $\textsl{OPS}$ of problem (\ref{eq10}) is an upper bound of the loss $\textsl{OP}$ incurred by problem (\ref{eq8}). Hence we can choose to optimize problem (\ref{eq10}) for efficiency, as only one QP problem is involved. Such problem can be easily solved by any off-the-shelf QP tools.

After the completion of the optimization process, the predicted label $\widetilde{y}$ of the text instance $\widetilde{\mathbf{x}}$ is given as:
\begin{align}
\label{eq11}
\widetilde{y} = \argmax\limits_{j\in [l]} \sum_{i=1}^ma_{ij}\kappa(\widetilde{\mathbf{x}}, \mathbf{x}_i) + b_j
\end{align}
The pseudo code of SURE is presented in Algorithm 1.
\begin{algorithm}[t]
\caption{The SURE Algorithm}
\label{alg1}
\begin{algorithmic}[1]
\REQUIRE\mbox{}\par
$\mathcal{D}$: the PL training set $\mathcal{D}=\{(\mathbf{X}, \mathbf{Y})\}$\\
$\lambda, \beta$: the regularization parameters\\
$\widetilde{\mathbf{x}}$: the unseen test instance\\
\ENSURE \mbox{}\par
$\widetilde{y}$: the predicted label for the test instance $\widetilde{\mathbf{x}}$\\
\item[]
\STATE construct the kernel matrix $\mathbf{K}=[\kappa(\mathbf{x}_i, \mathbf{x}_j)]_{m\times m}$;
\STATE initialize $\mathbf{P}=\mathbf{Y}$;
\REPEAT
\STATE update $\mathbf{A}=[a_{ij}]_{m\times l}$ and $\mathbf{b}=[b_j]_l$ according to (\ref{eq6});
\STATE update $\mathbf{Q} = \mathbf{KA} + \mathbf{1}\mathbf{b}^\top$;
\STATE calculate $\mathbf{P}$ by solving (\ref{eq10}) with a general QP procedure for each training example;
\UNTIL convergence or the maximum number of iterations.\STATE return the predicted label $\widetilde{y}$ according to (\ref{eq11}).
\end{algorithmic}
\end{algorithm}
\section{Experiments}
\subsection{Comparing Algorithms}
To demonstrate the effectiveness of SURE, we conduct extensive experiments to compare SURE with six state-of-the-art partial label learning algorithms, each configured with suggested parameters according to the respective literature:
\begin{itemize}
\item PLKNN~\cite{hullermeier2006learning}: a $k$-nearest neighbor approach that makes predictions by averaging the labeling information of neighboring examples [suggested configuration: $k\in\{5,6,\cdots,10\}$];
\item CLPL~\cite{cour2011learning}: a convex formulation that deals with PL examples by transforming partial label learning problem to binary learning problem via feature mapping [suggested configuration: SVM with squared hinge loss];
\item IPAL~\cite{zhang2015solving}: an instance-based approach that disambiguates candidate labels by an adapted label propagation scheme. [suggested configuration: $\alpha\in\{0,0.1,\cdots,1\}$, $k\in\{5,6,\cdots,10\}$];
\item PLSVM~\cite{nguyen2008classification}: a maximum margin approach that learns from PL examples by optimizing margin-based objective function [suggested configuration: $\lambda\in\{10^{-3},10^{-2},\cdots,10^3\}$];
\item PALOC~\cite{wu2018towards}: an approach that adapts one-vs-one decomposition strategy to enable binary decomposition for learning from PL examples [suggested configuration: $\mu=10$];
\item LSBCMM~\cite{liu2012conditional}: a maximum likelihood approach that learns from PL examples via mixture models [suggested configuration: $L = \lceil 10\log_2(l) \rceil$].
\end{itemize}
The two parameters $\lambda$ and $\beta$ for SURE 
 are chosen from $\{0.001, 0.01, 0.05, 0.1, 0.3, 0.5, 1\}$. Parameters for each algorithm are selected by five-fold cross-validation on the training set. For each dataset, ten-fold cross-validation is performed where mean prediction accuracies and the standard deviations are recorded. In addition, we use $t$-test at 0.05 significance level for two independent samples to investigate whether SURE is significantly superior/inferior (win/loss) to the comparing algorithms for all the experiments.
\subsection{Controlled UCI Datasets}
The characteristics of 4 controlled UCI datasets are reported in Table 1. Following the widely-used controlling protocol~\cite{cour2011learning,liu2012conditional,zhang2015solving,wu2018towards,feng2018leveraging,wang2018towards}, each UCI dataset can be used to generate artificial partial label datasets. There are three controlling parameters $p$, $r$ and $\epsilon$ where $p$ controls the proportion of PL examples, $r$ controls the number of false positive labels, and $\epsilon$ controls the probability of a specific false positive label occurring with the ground-truth label. As shown in Table 1, there are 4 configurations, each corresponding to 7 results. Hence we can totally generate $4\times 4\times 7=112$ different artificial partial label datasets.

Figure 1 shows the classification accuracy of each algorithm as $\epsilon$ ranges from 0.1 to 0.7 when $p=0.1$ and $r=1$ (Configuration (\uppercase\expandafter{\romannumeral1})). In this setting, a specific label is selected as the coupled label that co-occurs with the ground-truth label with probability $\epsilon$, and any other label would be randomly chosen to be a false positive label with probability $1-\epsilon$. Figures 2, 3, and 4 illustrate the classification accuracy of each algorithm as $p$ ranges from 0.1 to 0.7 when $r=1,2$, and 3 (Configuration (\uppercase\expandafter{\romannumeral2}), (\uppercase\expandafter{\romannumeral3}), and (\uppercase\expandafter{\romannumeral4})), respectively. In these three settings, $r$ extra labels are randomly chosen to be the false positive labels. That is, the number of candidate labels for each instance is $r+1$.
\begin{table}[!t]
\centering\caption*{Table 1: Characteristics of the controlled UCI datasets.}
\setlength{\tabcolsep}{3mm}{
\begin{tabular}{c|c|c|c|c}
\hline
\hline
Dataset & deter & ecoli & glass & usps \\
\hline
Examples  & 358 & 336 & 214 & 9298 \\
\hline
Features & 23 & 7 & 9 & 256\\
\hline
Labels & 6 & 8 & 6 & 10\\
\hline
\hline
\end{tabular}
Configurations:\\
(\uppercase\expandafter{\romannumeral1}) $p=0.1,r=1$, $\epsilon\in\{0.1,0.2,\cdots,0.7\}$\\
(\uppercase\expandafter{\romannumeral2}) $r=1$, $p\in\{0.1,0.2,\cdots,0.7\}$\\
(\uppercase\expandafter{\romannumeral3}) $r=2$, $p\in\{0.1,0.2,\cdots,0.7\}$\\
(\uppercase\expandafter{\romannumeral4}) $r=3$, $p\in\{0.1,0.2,\cdots,0.7\}$
}
\label{tab1}
\end{table}
\begin{figure*}[!t]
\subfigure[ecoli]{
    \begin{minipage}[b]{0.24\textwidth}
      \centering
      \includegraphics[width=1.7in]{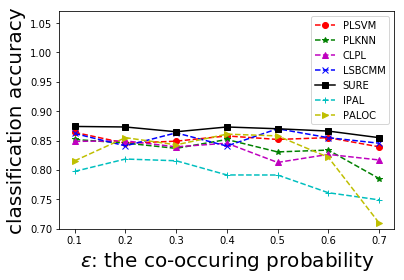}
    \end{minipage}}%
  \subfigure[deter]{
    \begin{minipage}[b]{0.24\textwidth}
      \centering
      \includegraphics[width=1.7in]{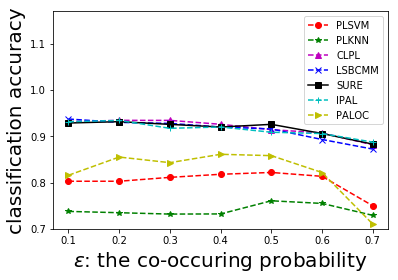}
    \end{minipage}}
      \subfigure[glass]{
    \begin{minipage}[b]{0.24\textwidth}
      \centering
      \includegraphics[width=1.7in]{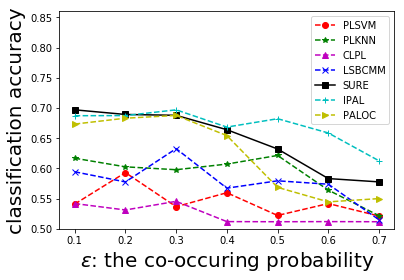}
    \end{minipage}}
      \subfigure[usps]{
    \begin{minipage}[b]{0.24\textwidth}
      \centering
      \includegraphics[width=1.7in]{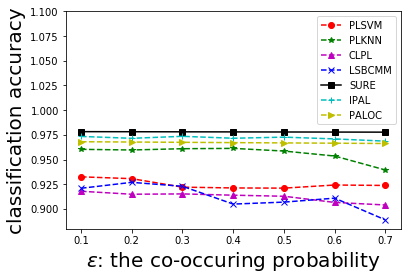}
    \end{minipage}}
  \caption*{Figure 1: Classification performance on controlled UCI datasets with $\epsilon$ ranging from 0.1 to 0.7 ($p=1,r=1$).}
  \label{fig1} 
\end{figure*}
\begin{figure*}[!t]
    \addtocounter{subfigure}{-4}\subfigure[ecoli]{
    \begin{minipage}[b]{0.24\textwidth}
      \centering
      \includegraphics[width=1.7in]{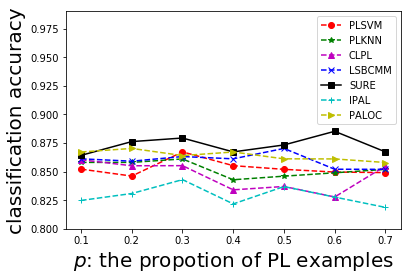}
    \end{minipage}}%
   \subfigure[deter]{
    \begin{minipage}[b]{0.24\textwidth}
      \centering
      \includegraphics[width=1.7in]{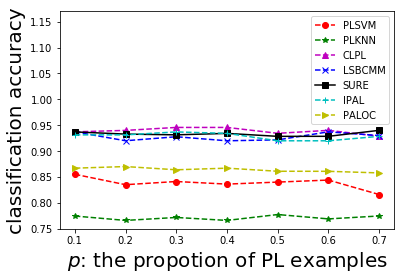}
    \end{minipage}}
   \subfigure[glass]{
    \begin{minipage}[b]{0.24\textwidth}
      \centering
      \includegraphics[width=1.7in]{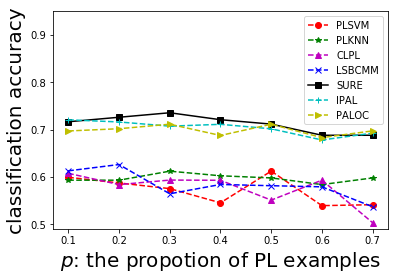}
    \end{minipage}}
    \subfigure[usps]{
    \begin{minipage}[b]{0.24\textwidth}
      \centering
      \includegraphics[width=1.7in]{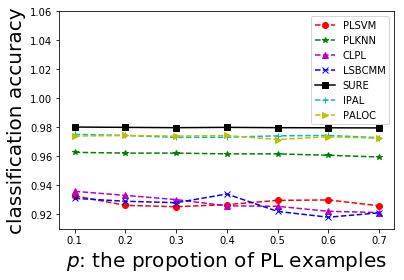}
    \end{minipage}}
  \caption*{Figure 2: Classification performance on controlled UCI datasets with $p$ ranging from 0.1 to 0.7 ($r=1$).}
  \label{fig2} 
\end{figure*}
\begin{figure*}[!t]
  \addtocounter{subfigure}{-4}\subfigure[ecoli]{
    \begin{minipage}[b]{0.24\textwidth}
      \centering
      \includegraphics[width=1.7in]{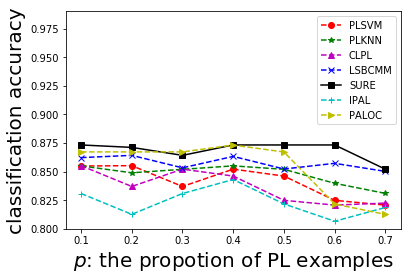}
    \end{minipage}}%
  \subfigure[deter]{
    \begin{minipage}[b]{0.24\textwidth}
      \centering
      \includegraphics[width=1.7in]{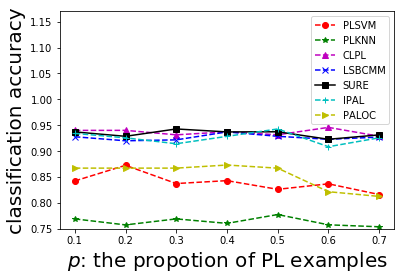}
    \end{minipage}}
      \subfigure[glass]{
    \begin{minipage}[b]{0.24\textwidth}
      \centering
      \includegraphics[width=1.7in]{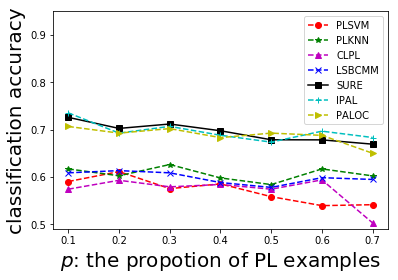}
    \end{minipage}}
      \subfigure[usps]{
    \begin{minipage}[b]{0.24\textwidth}
      \centering
      \includegraphics[width=1.7in]{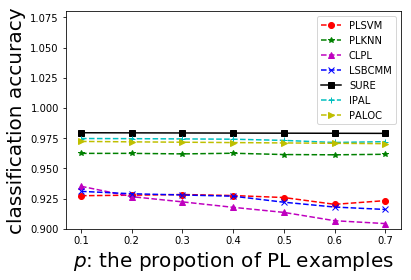}
    \end{minipage}}
  \caption*{Figure 3: Classification performance on controlled UCI datasets with $p$ ranging from 0.1 to 0.7 ($r=2$).}
  \label{fig3} 
\end{figure*}
\begin{figure*}[!t]
  \addtocounter{subfigure}{-4}\subfigure[ecoli]{
    \begin{minipage}[b]{0.24\textwidth}
      \centering
      \includegraphics[width=1.7in]{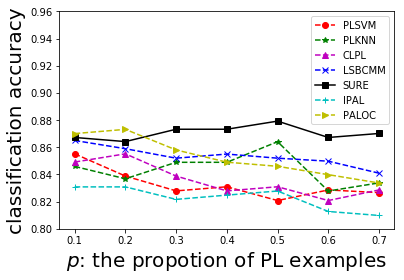}
    \end{minipage}}%
  \subfigure[deter]{
    \begin{minipage}[b]{0.24\textwidth}
      \centering
      \includegraphics[width=1.7in]{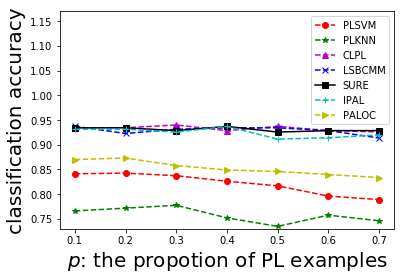}
    \end{minipage}}
      \subfigure[glass]{
    \begin{minipage}[b]{0.24\textwidth}
      \centering
      \includegraphics[width=1.7in]{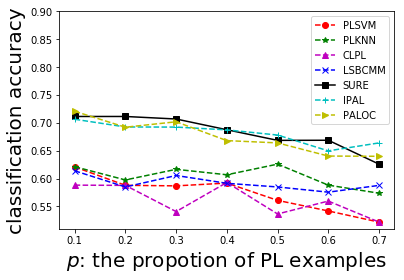}
    \end{minipage}}
      \subfigure[usps]{
    \begin{minipage}[b]{0.24\textwidth}
      \centering
      \includegraphics[width=1.7in]{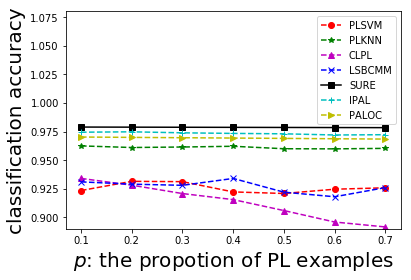}
    \end{minipage}}
  \caption*{Figure 4: Classification performance on controlled UCI datasets with $p$ ranging from 0.1 to 0.7 ($r=3$).}
  \label{fig4} 
\end{figure*}
As shown in Figures 1, 2, 3, and 4, SURE outperforms other comparing algorithms in general. To further statistically compare SURE with other algorithms, the detailed win/tie/loss counts between SURE and the comparing algorithms are recorded in Table 2. Out of the 112 results, it is easy to observe that:
\begin{itemize}
\item SURE achieves superior or at least comparable performance against PLKNN and PLSVM in all cases.
\item SURE achieves superior performance against CLPL and LSCMM in 72.3\% and 58.9\% cases while outperformed by them in only 4.5\% and 1.8\% cases, respectively.
\item SURE outperforms IPAL and PALOC in 50.9\% and 63.4\% cases while outperformed by them in only 5.4\% and 2.7\% cases, respectively.
\end{itemize}
In summary, the effectiveness of SURE on controlled UCI datasets is demonstrated.
\begin{table}[!t]
\centering\caption*{Table 2: Win/tie/loss ($t$-test at 0.05 significance level for two independent samples) counts on the controlled UCI datasets between SURE and the comparing algorithms.}
\setlength{\tabcolsep}{0.4mm}{
\begin{tabular}{c|c|c|c|c|c|c}
\hline
\hline
 & PLKNN & CLPL & IPAL & PLSVM & PALOC & LSBCMM \\
\hline
(\uppercase\expandafter{\romannumeral1}) & 24/4/0 & 21/6/1 & 14/11/3 & 21/7/0 & 20/8/0 & 14/14/0\\
\hline
(\uppercase\expandafter{\romannumeral2})  & 24/4/0 & 19/7/2 & 15/13/0 & 24/4/0 & 17/10/1 & 17/10/1\\
\hline
(\uppercase\expandafter{\romannumeral3})  & 24/4/0 & 21/6/1 & 14/13/1 & 25/3/0 & 17/10/1 & 19/9/0\\
\hline
(\uppercase\expandafter{\romannumeral4})  & 25/3/0 & 20/7/1 & 14/12/2 & 27/1/0 & 17/10/1 & 16/11/1\\
\hline
Total & \textbf{97/15/0} & \textbf{81/26/5} & \textbf{57/49/6} & \textbf{97/15/0} & \textbf{71/38/3} & \textbf{66/44/2}\\
\hline
\hline
\end{tabular}
}
\label{tab2}
\end{table}
\begin{table*}[!t]
\centering\caption*{Table 3: Characteristics of real-world partial label datasets.}
\label{tab3}
\setlength{\tabcolsep}{1.5mm}{
\begin{tabular}{c|c|c|c|c|c}
\hline
\hline
Dataset & Examples & Features & Labels & Avg. CLs & Task Domain\\
\hline
Lost & 1122  & 108 & 16 & 2.23 & \textsl{automatic face naming}~\cite{panis2014overview}\\
\hline
MSRCv2 & 1758 & 48 & 23 & 3.16 & \textsl{object classification}~\cite{liu2012conditional}\\
\hline
Soccer Player & 17472 & 279 & 171 & 2.09 & \textsl{automatic face naming}~\cite{zeng2013learning}\\
\hline
Yahoo! News & 22991 & 163 & 219 & 1.91 & \textsl{automatic face naming}~\cite{guillaumin2010multiple}\\
\hline
FG-NET & 1002 & 262 & 78 & 7.48 & \textsl{facial age estimation}~\cite{panis2014overview}\\
\hline
\hline
\end{tabular}
}
\end{table*}
\begin{table*}[!t]
\centering\caption*{Table 4: Classification accuracy of each algorithm on the real-world datasets. Furthermore, $\bullet$/$\circ$ indicates whether SURE is statistically superior/inferior to the comparing algorithm ($t$-test at 0.05 significance level for two independent samples).}
\label{tab4}
\setlength{\tabcolsep}{0.7mm}{
\begin{tabular}{c|c|c|c|c|c|c|c}
\hline
\hline
 & SURE & PLKNN & CLPL & IPAL & PLSVM & PALOC & LSBCMM\\
\hline
Lost & 0.781$\pm$0.039 & 0.432$\pm$0.051$\bullet$ & 0.742$\pm$0.038$\bullet$ & 0.678$\pm$0.053$\bullet$ & 0.729$\pm$0.042$\bullet$ & 0.629$\pm$0.056$\bullet$ & 0.693$\pm$0.035$\bullet$\\
\hline
MSRCv2 & 0.515$\pm$0.027 & 0.417$\pm$0.034$\bullet$ & 0.413$\pm$0.041$\bullet$ & 0.529$\pm$0.039 & 0.461$\pm$0.046$\bullet$ & 0.479$\pm$0.042$\bullet$ & 0.473$\pm$0.037$\bullet$\\
\hline
Soccer Player & 0.533$\pm$0.017 & 0.495$\pm$0.018$\bullet$ & 0.368$\pm$0.010$\bullet$ & 0.541$\pm$0.016 & 0.464$\pm$0.011$\bullet$ & 0.537$\pm$0.015 & 0.498$\pm$0.017$\bullet$\\
\hline
Yahoo! News & 0.644$\pm$0.015 & 0.483$\pm$0.011$\bullet$ & 0.462$\pm$0.009$\bullet$ & 0.609$\pm$0.011$\bullet$ & 0.629$\pm$0.010$\bullet$ & 0.625$\pm$0.005$\bullet$ & 0.645$\pm$0.005\\
\hline
FG-NET & 0.078$\pm$0.021 & 0.039$\pm$0.018$\bullet$ & 0.063$\pm$0.027 & 0.054$\pm$0.030$\bullet$ & 0.063$\pm$0.029 & 0.065$\pm$0.019 & 0.059$\pm$0.025$\bullet$\\
\hline
FG-NET(MAE3) & 0.458$\pm$0.024 & 0.269$\pm$0.045$\bullet$ & 0.458$\pm$0.022 & 0.362$\pm$0.034$\bullet$ & 0.356$\pm$0.022$\bullet$ & 0.435$\pm$0.018$\bullet$ & 0.382$\pm$0.029$\bullet$\\
\hline
FG-NET(MAE5) & 0.615$\pm$0.019 & 0.438$\pm$0.053$\bullet$ & 0.596$\pm$0.017$\bullet$ & 0.540$\pm$0.033$\bullet$ & 0.479$\pm$0.016$\bullet$ & 0.609$\pm$0.043 & 0.532$\pm$0.038$\bullet$\\
\hline
\hline
\end{tabular}
}
\end{table*}
\subsection{Real-World Datasets}
Table 3 reports the characteristics of real-world partial label datasets\footnote{These datasets are publicly available at: \url{http://cse.seu.edu.cn/PersonalPage/zhangml/}}
including Lost~\cite{cour2011learning}, MSRCv2~\cite{liu2012conditional}, Soccer Player~\cite{zeng2013learning}, Yahoo! News~\cite{guillaumin2010multiple}, and FG-NET~\cite{panis2014overview}. These real-world partial label datasets are from several task domains. For \textsl{automatic face naming} (Lost, Soccer Player, and Yahoo! News), each face (instance) is cropped from an image or a video frame, and the names appearing on the corresponding captions or subtitles are taken as candidate labels. For \textsl{facial age estimation} (FG-NET), human faces are regarded as instances while ages annotated by crowdsourcing labelers serve as candidate labels. For \textsl{object classification} (MSRCv2), each image segment is considered as an instance, and objects appearing in the same image are taken as candidate labels. The average number of candidate labels (Avg. CLs) per instance is also reported in Table 3.

Table 4 reports the mean classification accuracy as well as the standard deviation of each algorithm on each real-world dataset. Note that the average number of candidate labels (Avg. CLs) of FG-NET dataset is quite large, which results in an extremely low classification accuracy of each algorithm. For better evaluation of this facial age estimation task, we employ conventional mean absolute error (MAE)~\cite{zhang2016partial} to conduct two extra experiments. Specifically, for FG-NET (MAE3/MAE5), a test example is considered correctly classified if the MAE between the predicted age and the ground-truth age is no more than 3/5 years. As shown in Table 4, we can observe that:
\begin{itemize}
\item SURE significantly outperforms PLKNN on all the real-world datasets.
\item Out of the 42 cases (6 comparing algorithms and 7 datasets), SURE significantly outperforms all the comparing algorithms in 78.6\% cases, and achieves competitive performance in 21.4\% cases.
\item It is worth noting that SURE is never significantly outperformed by any comparing algorithms.
\end{itemize}
These experimental results on real-world datasets also demonstrate the effectiveness of SURE.
\subsection{Further Analysis}
\subsubsection{Parameter Sensitivity Analysis}
\begin{figure}[!t]
\centering
\addtocounter{subfigure}{-4}
\subfigure[Varying $\lambda$ on Lost]{\label{fig:subfig:a}
\includegraphics[width=0.45\linewidth]{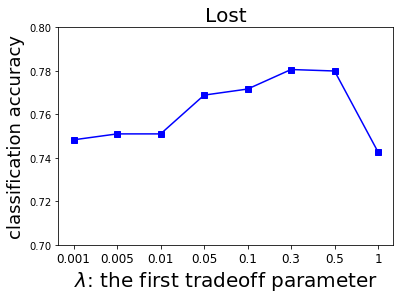}}
\hspace{0.01\linewidth}
\subfigure[Varying $\beta$ on Lost]{\label{fig:subfig:b}
\includegraphics[width=0.45\linewidth]{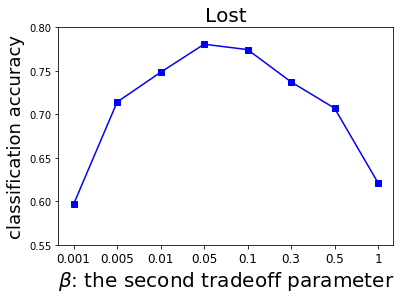}}
\vfill
\subfigure[Convergence on Lost]{\label{fig:subfig:c}
\includegraphics[width=0.45\linewidth]{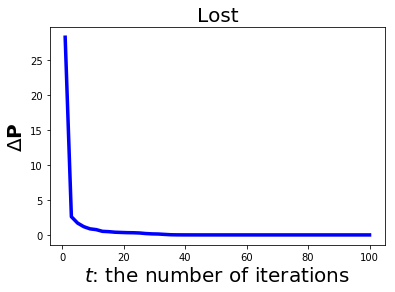}}
\hspace{0.01\linewidth}
\subfigure[Convergence on MSRCv2]{\label{fig:subfig:d}
\includegraphics[width=0.45\linewidth]{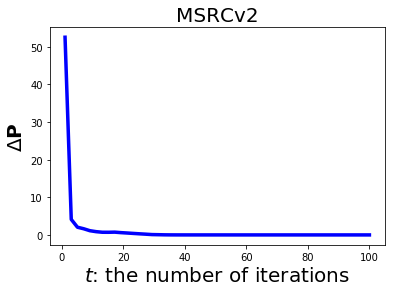}}
\caption*{Figure 5: Parameter sensitivity and convergence analysis for SURE. (a) Sensitivity analysis of $\lambda$ on Lost; (b) Sensitivity analysis of $\beta$ on MSRCv2; (c) Convergence analysis on Lost; (d) Convergence analysis on MSRCv2.}
\label{fig:subfig}
\end{figure}

There are two tradeoff parameters $\lambda$ and $\beta$ for SURE, which should be manually searched in advance. Hence this section studies how $\lambda$ and $\beta$ influence the prediction accuracy produced by SURE. We vary one parameter, while keeping the other fixed at the best setting. Figures 5(a) and 5(b) show the performance of SURE on the Lost dataset given different values of $\lambda$ and $\beta$ respectively. Note that $\lambda$ controls the importance of the maximum infinity norm regularization. When $\lambda$ is very small, the mutually exclusive relationships among labels are hardly considered, thus the classification accuracy would be at a low level. As $\lambda$ increases, we start to take into consideration such exclusive relationships, and the classification accuracy increases. However, if $\lambda$ is sufficiently large, the classification accuracy will drop dramatically. This is because when we overly concentrate on the mutually exclusive relationships among labels, we will directly regard the candidate label that has the maximal modeling output as the ground-truth label. Since to maximize the infinity norm $\left\|\mathbf{p}\right\|_{\infty}$ is overly important, the approximation loss will be totally ignored. From the above, we can draw a conclusion that it would be better to balance the approximation loss and the mutually exclusive relationships among labels. Such conclusion clearly comfirms the effectiveness of the SURE approach. Another tradeoff parameter $\beta$ aims to control the model complexity. The classification accuracy curve of varying $\beta$ obviously accords with our cognition that it is important to balance between overfitting and underfitting.
\subsubsection{Illustration of Convergence}
We illustrate the convergece of SURE by using the difference of the optimization variable $\mathbf{P}$ between two successive iterations ($\Delta\mathbf{P} = \left\|\mathbf{P}^{(t+1)}-\mathbf{P}^{(t)}\right\|_F$). Figure 5(c) and 5(d) show the convergence curves of SURE on Lost and MSRCv2 respectively. It is apparent that $\Delta\mathbf{P}$ gradually decreases to 0 as the number of iterations $t$ increases. Hence the convergence of SURE is demonstrated.
\section{Conclusion}
In this paper, we utilize the idea of self-training to exaggerate the mutually exclusive relationships among candidate labels for further enhancing partial label learning performance. Instead of manually performing pseudo-labeling after model training, we propose a unified formulation (named SURE) with the maximum infinity norm regularization to train the desired model and perform pseudo-labeling jointly. Extensive experimental results demonstrate the effectiveness of SURE.

Since self-training is a typical semi-supervised learning method, it would be interesting to extend SURE to the setting of semi-supervised learning. Besides, as mutually exclusive relationships exist in general multi-class problems, it would be valuable to explore other possible ways to incorporate such relationships into partial label learning.
\section{Acknowledgements}
This work was supported by MOE, NRF, and NTU.


\bibliographystyle{aaai}
\bibliography{aaai19}

\end{document}